\documentclass[11pt,journal,compsoc]{IEEEtran}
\usepackage[utf8]{inputenc}
\usepackage{amsmath, amssymb, bm}
\usepackage{graphicx}
\usepackage{booktabs}
\usepackage{algorithm}
\usepackage{algpseudocode}
\usepackage{cite}
\usepackage{hyperref}

\newenvironment{proof}{\begin{IEEEproof}}{\end{IEEEproof}}

\newtheorem{lemma}{Lemma}
\newtheorem{proposition}{Proposition}

\begin{document}

\title{Edge-Centric Adaptive Neural Activation via Parametric Tanh Functions}
\author{Andreas~Falkenberg%
\thanks{A. Falkenberg is with QPI Laboratories, Inc., USA (e-mail: andreas@qpilabs.com).}}

\IEEEtitleabstractindextext{%
\begin{abstract}
This paper introduces a compact, edge-centric neural architecture that resolves the computational bottlenecks of modern edge-activated paradigms like Kolmogorov-Arnold Networks (KANs) while retaining their expressiveness over traditional Multi-Layer Perceptrons (MLPs). While KANs demonstrate remarkable representational capacity by placing learnable univariate functions on connection edges rather than hidden nodes, their reliance on non-local B-splines introduces severe optimization hurdles, high memory overhead, and poor hardware scaling. To bridge this gap, we propose the Parametric Tanh Connection Function, a smooth, bounded architectural primitive defined as $a_{\text{tanh}}(x, \delta, \alpha, \beta, \gamma) = \delta + \alpha \tanh(\beta x + \gamma)$. This formulation unifies offset, amplitude, local slope, and horizontal displacement into a lightweight, 4-parameter differentiable mapping optimized edge-by-edge via backpropagation. Evaluated on dense architectures using MNIST and CIFAR-10, our parametric tanh connection model achieves significant absolute accuracy gains over traditional affine baselines—reaching 95.63\% $\pm$ 0.08\% on MNIST and 44.74\% $\pm$ 0.18\% on CIFAR-10. Crucially, it matches or exceeds the expressiveness of unconstrained 7-parameter exponential connection baselines while requiring up to 85\% less training runtime. By replacing complex splines or hyper-flexible exponents with a bounded, analytically smooth hyperbolic tangent kernel, this work delivers a highly scalable, hardware-friendly compromise for heterogeneous connection-level deep learning.
\end{abstract}

}

\maketitle
\IEEEdisplaynontitleabstractindextext
\IEEEpeerreviewmaketitle

\section{Introduction}
Conventional neural networks attach a scalar weight to each connection and apply a shared nonlinear activation after the weighted inputs have been combined. This separation is computationally convenient, but it limits the learnable behavior of an individual connection to affine scaling. A connection-level function instead permits each input branch to learn its own nonlinear response before aggregation. The resulting model can adapt not only the magnitude and offset of a connection, but also its local slope and saturation behavior through the same optimization process.

The primary contributions of this work are summarized as follows:
\begin{itemize}
    \item The proposed parametric tanh function effectively learns offset, amplitude, slope, and horizontal displacement on each unique connection edge.
    \item Connection-level placement permits heterogeneous nonlinear responses within the same layer rather than imposing a single post-aggregation activation function on all hidden nodes.
    \item The seven-parameter analytic family motivates the design space and serves as an ablation for assessing whether additional functional flexibility improves empirical performance.
    \item Benchmarks on MNIST and CIFAR-10 demonstrate that the four-parameter tanh model provides the highest test accuracy among the three evaluated configurations while training substantially faster than the seven-parameter variant.
\end{itemize}

The remainder of this paper is structured as follows. Section 2 establishes conventional neural network and optimization notation required for the connection-level formulation. Section 3 formalizes single-layer neural classification, direct least-squares, and gradient-based iterative methods. Section 4 introduces the seven-parameter analytic connection family, derives the compact parametric tanh model alongside its backpropagation derivatives, and presents the unified layer learning algorithms. Section 5 evaluates the proposed configurations empirically against an affine baseline using dense and convolutional topologies on MNIST and CIFAR-10. Section 6 provides concluding remarks and future research directions.

\section{Background and Problem Formulation}

\subsection{Prerequisite}
For a general classification task, the input is a real-valued feature vector and the classifier maps that vector to a discrete class label~\cite{rosenblatt1958perceptron}. Complex-valued observations can be represented by concatenating their real and imaginary components into an equivalent real-valued vector:
\begin{equation}
H: \mathbb{R}^n \to \{0, 1, \dots, k-1\},
\end{equation}
with $\bm{x} \in \mathbb{R}^n$ and $H(\bm{x}) \in \{0, 1, \dots, k-1\}$. The output result is an integer between $0$ and $k-1$. The function $H$ maps a real-valued input feature vector from a sensor or dataset to an output class represented by that integer label. We can construct neural networks for this mapping using column-vector notation.

Let $\bm{x} \in \mathbb{R}^n$ denote the layer input and $\bm{y} \in \mathbb{R}^k$ the layer output. The weight matrix is defined as $\bm{W} \in \mathbb{R}^{k \times n}$ and the affine mapping is given by $\bm{y} = \bm{W}\bm{x} + \bm{b}$, where $\bm{b} \in \mathbb{R}^k$. If row-vector notation is used instead, the transpose convention must be applied consistently throughout; this paper strictly adheres to column vectors:
\begin{equation}
[\bm{W} \quad \bm{b}] \begin{bmatrix} \bm{x}^{\top} \\ 1 \end{bmatrix} = \bm{W}\bm{x} + \bm{b}.
\end{equation}

In a conventional neural network, each neuron typically includes an additive bias. The same affine mapping can be represented in homogeneous coordinates by augmenting the input vector with a constant component equal to one. The bias therefore does not change the underlying mathematical class, although it remains important in implementation. Each matrix element represents a trainable scalar weight, and the final predicted class is obtained from the maximum output score. For a one-layer model, the classifier can be written as follows:
\begin{equation}
\bm{z} = \bm{W}\bm{x} + \bm{b}.
\end{equation}

For a one-layer classifier, the network first computes a vector of $k$ real-valued scores $\bm{z} \in \mathbb{R}^k$. The predicted class is then:
\begin{equation}
H(\bm{x}) = \arg\max_{1 \le j \le k} z_j.
\end{equation}
The $\arg\max$ operation is used exclusively for inference and is not treated as an activation function during gradient-based training; a differentiable loss is applied to the score vector $\bm{z}$ instead.

For an $m$-layer classifier, the network transitions are formalized layer by layer. Let $\ell \in \{1, \dots, m\}$ denote the layer index:
\begin{equation}
\bm{z}^{(\ell)} = \bm{W}^{(\ell)}\bm{h}^{(\ell-1)} + \bm{b}^{(\ell)}, \quad \bm{h}^{(\ell)} = \phi^{(\ell)}(\bm{z}^{(\ell)}).
\end{equation}
The trainable network produces a real-valued score vector at the final output layer, and the predicted class is selected by an $\arg\max$ operation applied only during inference. The $m-1$ hidden transitions may use nonlinear activation functions $\phi^{(\ell)}$, while the output layer may remain linear when the training objective is applied directly to the score vector. Matrix dimensions must satisfy standard compatibility conditions: the output dimension of each layer equals the input dimension of the next layer, the first layer accepts the $n$-dimensional input vector, and the final layer produces $k$ class scores:
\begin{equation}
H(\bm{x}) = \arg\max_{1 \le j \le k} z_j^{(m)}.
\end{equation}

If every hidden activation is the identity, consecutive affine transformations can be composed into a single affine map. Depth alone therefore does not increase representational capacity in a purely linear network. This observation motivates the use of nonlinear hidden transformations.

\begin{lemma}
A composition of affine layers with identity activations can be replaced by a single affine layer. In particular, if $\bm{h}^{(\ell)} = \bm{W}^{(\ell)}\bm{h}^{(\ell-1)} + \bm{b}^{(\ell)}$ and every activation is the identity, then the network is equivalent to $\bm{h}^{(m)} = \bm{W}_{\text{eq}}\bm{x} + \bm{b}_{\text{eq}}$ for suitable $\bm{W}_{\text{eq}}$ and $\bm{b}_{\text{eq}}$.
\end{lemma}

Consequently, the analysis focuses on nonlinear activation functions. A network containing only affine layers cannot represent nonlinearly separable decision boundaries such as XOR, whereas a multilayer network with at least one nonlinear hidden transformation can do so~\cite{cybenko1989approximation}. The input dimension is determined by the feature vector and the output dimension by the number of classes. The following homogeneous-coordinate result permits affine layers to retain a clean matrix representation.

\begin{lemma}
An affine layer with $n$ inputs and $k$ outputs, $\bm{y} = \bm{W}\bm{x} + \bm{b}$, can be represented as a matrix multiplication in homogeneous coordinates by augmenting the input with a constant one. Specifically, define $\tilde{\bm{x}} = [\bm{x}^{\top}, 1]^{\top} \in \mathbb{R}^{n+1}$ and $\tilde{\bm{W}} = [\bm{W} \quad \bm{b}] \in \mathbb{R}^{k \times (n+1)}$, such that $\bm{y} = \tilde{\bm{W}}\tilde{\bm{x}}$.
\end{lemma}
If the augmented constant must be propagated to a subsequent layer, append an additional constant output separately. The trainable layer itself still produces $k$ outputs, not $k+1$.

\subsection{Overview of the Following Chapters}
The following sections progress from conventional affine models and gradient-based optimization to learned nonlinear connection functions. The seven-parameter formulation is introduced first because it exposes a broad analytic design space and clarifies how weight-like and activation-like behavior can be unified on an edge. That generality also introduces trainable-exponent domain restrictions and substantial computational cost. These considerations motivate the four-parameter tanh connection function, which retains connection-level nonlinear adaptation while providing a smooth bounded core, unrestricted real inputs, straightforward derivatives, and a materially lower training cost than the full formulation.

\section{Classification with a Single Layer of Neurons}

\subsection{Problem Definition}
Before defining the learning and matrix-adjustment strategy, the inputs must be specified. Each sample is represented by multidimensional input data from a sensor or dataset and flattened into a real-valued input feature vector. The total number of inputs is the product of the data dimensions; additional dimensions, such as repeated observations, can be incorporated in the same way.

Let the flattened input at sample index $t$ be denoted by $\bm{x}(t) \in \mathbb{R}^n$. When homogeneous coordinates are used, the constant component introduced to represent the bias is omitted from the notation below unless it is required explicitly: $\bm{x}(t) = [x_1(t), \dots, x_n(t)]^{\top}$. Further, assume that the classification task has $k$ possible output classes, represented in the output vector:
\begin{equation}
\bm{y}(t) = [y_1(t), \dots, y_k(t)]^{\top}.
\end{equation}
Therefore, the coefficient matrix $\bm{M}$ is required to have dimensions $k \times n$, and the forward batch calculation is performed as follows:
\begin{equation}
\bm{y}(t) = \bm{M}\bm{x}(t).
\end{equation}
The entire matrix $\bm{M}$ needs to be optimized globally such that the expected target mapping is learned accurately from data.

\subsection{Defining the Matrix by Matrix Inversion}
For supervised classification, each training example consists of an input feature vector $\bm{x}^{(s)}$ and a target $\bm{y}^{(s)}$. The target may be represented as a class index for a cross-entropy objective or as a $k$-dimensional encoded vector for a squared-error objective. The paper uses the latter convention in the derivations below. The training set contains $r$ such input--target pairs:
\begin{equation}
\left\{ \left(\bm{x}^{(s)}, \bm{y}^{(s)}\right) \right\}_{s=1}^{r}.
\end{equation}
The coefficient matrix $\bm{M}$ is unknown and must be estimated from the training data. Let $\hat{\bm{M}}$ denote this estimate. The training input--target pairs can be assembled in matrix form to derive a least-squares estimate of $\hat{\bm{M}}$:
\begin{equation}
\bm{X} = \begin{bmatrix} \bm{x}^{(1)} & \dots & \bm{x}^{(r)} \end{bmatrix}^{\top}, \quad \bm{Y} = \begin{bmatrix} \bm{y}^{(1)} & \dots & \bm{y}^{(r)} \end{bmatrix}^{\top}.
\end{equation}
With $\bm{X} \in \mathbb{R}^{r \times n}$ and $\bm{Y} \in \mathbb{R}^{r \times k}$, the training samples are stored as rows. In the bias-free least-squares formulation, the estimated coefficient matrix satisfies $\hat{\bm{M}}^{\top} = \bm{X}^{+}\bm{Y}$, or equivalently, $\hat{\bm{M}} = \bm{Y}^{\top}(\bm{X}^{+})^{\top}$, where $\hat{\bm{M}} \in \mathbb{R}^{k \times n}$.

The single-sample model uses column vectors and is written as $\bm{y} = \bm{W}\bm{x} + \bm{b}$, with $\bm{W} \in \mathbb{R}^{k \times n}$. For $r$ samples arranged as rows of $\bm{X}$ and $\bm{Y}$, the corresponding batch equation is $\bm{Y} = \bm{X}\bm{W}^{\top} + \bm{1}\bm{b}^{\top}$, where $\bm{1} \in \mathbb{R}^r$ is the all-ones vector. In the bias-free case, the least-squares solution satisfies $\hat{\bm{W}}^{\top} = \bm{X}^{+}\bm{Y}$ or, equivalently, $\hat{\bm{W}} = \bm{Y}^{\top}(\bm{X}^{+})^{\top}$.

Because $\bm{X}^{+} \in \mathbb{R}^{n \times r}$, the product $\bm{X}^{+}\bm{Y} \in \mathbb{R}^{n \times k}$ represents the transpose of the column-vector weight matrix. The solution utilizes the Moore--Penrose pseudoinverse $\bm{X}^{+}$, rather than an ordinary matrix inverse of $\bm{X}$. Its computational cost and numerical stability depend directly on the dimensions, rank, conditioning, and numerical factorization used to compute the pseudoinverse.

\subsection{Error Calculation}
When direct least-squares estimation is impractical because the training matrix is large, rank deficient, or expensive to factorize, the parameters can instead be optimized iteratively from a differentiable loss. Let $\hat{\bm{y}}$ denote the model prediction for one sample and $\bm{y}$ its target. A squared-error objective may then be defined as follows:
\begin{equation}
L(\hat{\bm{y}}, \bm{y}) = \|\hat{\bm{y}} - \bm{y}\|_2^2 = \sum_{j=1}^{k} (\hat{y}_j - y_j)^2.
\end{equation}
This is the squared Euclidean norm, commonly used as a least-squares objective. Dividing the summed loss by the number of outputs produces the mean squared error and changes only its scale, not its minimizer. Partition the training set into mini-batches $\mathcal{B}_1, \dots, \mathcal{B}_m$. For a batch $\mathcal{B}$, define the batch loss as the mean of the sample losses over $s \in \mathcal{B}$:
\begin{equation}
L_{\mathcal{B}} = \frac{1}{|\mathcal{B}|} \sum_{s \in \mathcal{B}} L_s.
\end{equation}

Mini-batch optimization reduces the computational cost of each update relative to full-batch optimization. Very small batches may increase gradient variance, whereas very large batches increase memory demand and can alter optimization behavior. The batch size is therefore treated as an experimental hyperparameter.

\subsection{Optimization Methods}
The principal optimization methods considered here are direct least-squares estimation, first-order gradient descent, and second-order or curvature-aware optimization~\cite{nocedal2006optimization}. The proposed connection-level neuron does not constitute a separate optimization method; it is a different model parameterization that can be trained with first-order or higher-order methods.
\begin{itemize}
    \item Iterative Method (Pseudoinverse updates)
    \item Gradient Descent (First-order derivatives)
    \item Curvature-Aware Descent (Second-order derivatives)
    \item Connection-Level Optimization (Integrated weight-activation parameters)
\end{itemize}

\subsection{Iterative Method}
Assume an initial estimate $\bm{M}^{(0)}$ is selected. An iterative batch procedure computes the residual, estimates a correction from the current batch, updates the coefficient matrix, and repeats until a stopping criterion is satisfied, as formalized in Algorithm~\ref{alg:iterative_method}.

\begin{algorithm}
\caption{Iterative Pseudo-Inversion Matrix Adaptation}\label{alg:iterative_method}
\begin{algorithmic}
\State Set iteration index $t \gets 0$
\State Initialize coefficient matrix $\bm{M}^{(0)}$
\While{overall error is above threshold \textbf{and} $t < t_{\max}$}
    \State For the mini-batch $\mathcal{B}_t$, extract $\bm{X}_{\mathcal{B}_t}$ and $\bm{Y}_{\mathcal{B}_t}$
    \State Compute residual matrix: $\bm{R}^{(t)} \gets \bm{Y}_{\mathcal{B}_t} - \bm{X}_{\mathcal{B}_t}(\bm{M}^{(t)})^{\top}$
    \State Estimate batch correction: $\Delta \bm{M}^{(t)} \gets \left(\bm{X}_{\mathcal{B}_t}^{+} \bm{R}^{(t)}\right)^{\top}$
    \State Update weights: $\bm{M}^{(t+1)} \gets \bm{M}^{(t)} + \rho \Delta \bm{M}^{(t)}$ \quad ($\rho$: correction rate)
    \State $t \gets t + 1$
\EndWhile
\end{algorithmic}
\end{algorithm}

The input batch matrix must be inverted or factorized, which imposes a local algorithmic complexity. Selecting a sufficiently small batch size ensures this operation remains computationally manageable.

\subsection{Gradient Descent Method}
Gradient descent can optimize linear regression, logistic regression, deep linear networks, and nonlinear architectures alike. Nonlinear activation functions are required because a composition of affine layers remains an affine map and therefore cannot represent nonlinear decision boundaries. The limitation of a network with only linear activations is representational expressiveness, not the applicability of gradient-based optimization.

For each training pair $(\bm{x}^{(s)}, \bm{y}^{(s)})$, define the scalar loss $L_s = L(f(\bm{x}^{(s)}; \bm{M}), \bm{y}^{(s)})$. The Jacobian matrix gradient with respect to the weights is given by:
\begin{equation}
\nabla_{\bm{M}} L = \left[ \frac{\partial L}{\partial M_{i,j}} \right]_{\substack{i=1,\dots,k \\ j=1,\dots,n}}.
\end{equation}

For a selected input--target pair, each Jacobian entry measures the sensitivity of the scalar loss to one coefficient $M_{i,j}$. When an analytic derivative is unavailable, the entry can be approximated by perturbing that coefficient. Let $\bm{E}_{i,j} \in \mathbb{R}^{k \times n}$ denote an indicator matrix where $\bm{E}_{i,j}[r, c] = 1$ if $(r, c) = (i, j)$ and $0$ otherwise. Define the perturbed matrix as:
\begin{equation}
\bm{M}_{i,j}^{+} = \bm{M} + \varepsilon \bm{E}_{i,j}
\end{equation}
for a small scalar $\varepsilon > 0$. The derivative is approximated at a concrete input via finite differences:
\begin{equation}
\frac{\partial L}{\partial M_{i,j}} \approx \frac{L(\bm{M} + \varepsilon \bm{E}_{i,j}) - L(\bm{M})}{\varepsilon}.
\end{equation}

The factor $1/\varepsilon$ must be maintained explicit in numerical checks. Combining $\varepsilon$ algebraically with the learning rate defines an effective step size $\tilde{\eta} = \eta / \varepsilon$, which does not eliminate the requirement that $\varepsilon$ be balanced carefully to avoid numerical truncation and floating-point cancellation. The gradient descent update flow is formalized in Algorithm~\ref{alg:gradient_descent}.

\begin{algorithm}
\caption{First-Order Gradient Descent Batch Optimization}\label{alg:gradient_descent}
\begin{algorithmic}
\State Set iteration index $t \gets 0$
\State Initialize coefficient matrix $\bm{M}^{(0)}$
\While{overall error is above threshold \textbf{and} $t < t_{\max}$}
    \State For the mini-batch $\mathcal{B}_t$, compute batch gradient: $\bm{G}^{(t)} \gets \nabla_{\bm{M}} L_{\mathcal{B}_t}(\bm{M}^{(t)})$
    \State Set the correction matrix: $\Delta \bm{M}^{(t)} \gets -\bm{G}^{(t)}$
    \State Apply parameter update: $\bm{M}^{(t+1)} \gets \bm{M}^{(t)} + \eta \Delta \bm{M}^{(t)}$
    \State $t \gets t + 1$
\EndWhile
\end{algorithmic}
\end{algorithm}

\subsection{Complexity and Conclusion}
A rectangular least-squares problem is commonly solved by QR decomposition or singular-value decomposition rather than by explicitly forming an ordinary matrix inverse. Cholesky or $\bm{L}\bm{D}\bm{L}^{\top}$ factorization may be applied to suitable symmetric positive-definite systems, although forming normal equations directly can worsen conditioning. LU and Gauss--Jordan methods apply primarily to square systems. Consequently, the computational cost cannot be summarized by a single cubic bound without specifying matrix dimensions, rank assumptions, and the differentiation paradigm (analytic, automatic, or finite differences).

For an affine model, changing one matrix element affects the output through the corresponding input component, making the loss gradient inherently input-dependent. Finite-difference gradients are computationally expensive because each parameter perturbation requires an additional full forward loss evaluation, whereas analytic differentiation or automatic differentiation computes the complete gradient vector far more efficiently.

\begin{lemma}
If every activation function in a feed-forward network is linear or the identity, the composition of all layers is equivalent to a single affine transformation. Gradient descent can still optimize the resulting model, but adding depth does not increase its representational power or enable it to represent nonlinearly separable boundaries.
\end{lemma}

A multilayer extension replaces the single-layer update with backpropagation, which applies the chain rule recursively to propagate the loss gradient through each preceding layer~\cite{rumelhart1986learning}. Section 4 derives this recursion for the proposed connection-level nonlinear functions.

\section{Trainable Connection-Level Nonlinearities}

\subsection{Context within Trainable Activation Literature}
The proposed architecture belongs to the broader family of trainable activation models, but differs fundamentally by placing the trainable nonlinear function directly on each individual connection link rather than applying a single activation function to the aggregated output of a node~\cite{agostinelli2015learning, apicella2021survey, dubey2022activation}. 
Learning an activation shape alters both the representational capacity and model parameterization~\cite{apicella2021survey}. Recent evolutionary approaches such as PANGAEA~\cite{bingham2022discovering} discover functional forms automatically, while fixed parameterization scaling rules~\cite{jagtap2020adaptive,he2024local} alter the topology of the loss function to accelerate convergence.
Properties such as smoothness, curvature, range, and gradient flow heavily influence optimization trajectories and overall convergence rates. 
Earlier parametric activations typically introduced a small number of trainable degrees of freedom while retaining standard node-level placement (e.g., PReLU~\cite{he2015delving}). In contrast, Kolmogorov--Arnold Networks (KAN)~\cite{liu2025kan} place learnable univariate functions on edges and sum their outputs at nodes, demonstrating the modern relevance of connection-level nonlinear formulations. The present work adopts an edge-centric placement but focuses on a compact, smooth analytic tanh parameterization.

\subsection{Definition of the Activation Function}
Rather than applying an activation function after aggregation, the unique function associated with input link $x_i$ acts before the sum. The scalar weight normally attached to a connection link is replaced entirely by a parameter vector $\bm{p}$. Weight-like linear scaling and activation-like nonlinear curvature are represented within the same functional form and are updated together through a unified backpropagation step.

Given the parameter vector $\bm{p} = (\delta, \theta, \vartheta, \gamma, \alpha, \eta, \beta) \in \mathbb{R}^7$, we define the seven-parameter connection function as:
\begin{equation}
a_{\bm{p}}(x) = \delta + \theta x^{\vartheta} + \gamma \exp\left(\alpha x^{\eta} + \beta\right),
\end{equation}
where $\delta$ provides an additive offset, $\theta$ provides weight-like scaling, and $\gamma, \alpha, \eta, \beta$ control the amplitude and shape of the exponential term. The new neuron output $y_j$ sums these transformed inputs:
\begin{equation}
y_j = \sum_{i=1}^{n} a_{\bm{p}_{i,j}}(x_i).
\end{equation}

\begin{lemma}[Affine and identity representation]
A conventional scalar affine connection with bias can be represented exactly by the proposed seven-parameter function.
\end{lemma}
\begin{proof}
Set parameters $\delta = b$, $\theta = w$, $\vartheta = 1$, and $\gamma = 0$. The remaining exponential parameters may take arbitrary values since their term is scaled by zero. Substitution yields $a_{\bm{p}}(x) = b + wx$, exactly reproducing an affine weight $w$ and bias $b$. The identity function is recovered by setting $\delta = 0$, $\theta = 1$, $\vartheta = 1$, and $\gamma = 0$.
\end{proof}

\begin{proposition}[Sigmoid approximation]
The logistic sigmoid $\sigma(x) = \frac{1}{1 + e^{-x}}$ is not represented exactly by a finite parameter substitution in the proposed seven-parameter form.
\end{proposition}
The seven-parameter form contains a constant, a power term, and an exponential term, but lacks a reciprocal affine--exponential expression. It may, however, approximate a sigmoid on a bounded interval, subject to reporting an explicit interval and error bound $\max_{x \in [a,b]} |f(x) - \sigma(x)| \le \varepsilon$.

\begin{lemma}[Centered Gaussian representation]
A centered Gaussian of the form $g(x) = A \exp(-cx^2)$ with $c > 0$ can be represented exactly by the proposed function.
\end{lemma}
\begin{proof}
Set $\delta = 0$, $\theta = 0$, $\gamma = A$, $\alpha = -c$, $\eta = 2$, and $\beta = 0$; $\vartheta$ may be assigned any admissible value because its term is multiplied by $\theta = 0$. Substitution yields $a_{\bm{p}}(x) = A \exp(-cx^2) = g(x)$.
\end{proof}
A shifted Gaussian $A \exp(-c(x-\mu)^2)$ cannot be obtained directly from this parameterization because the exponent lacks a linear cross-term in $x$ unless the input itself is pre-shifted.

\begin{proposition}[Inverse square-root unit approximation]
The inverse square-root unit $f(x) = \frac{x}{\sqrt{1 + \lambda x^2}}$ is not represented exactly by fixing the seven parameters of the proposed function.
\end{proposition}
The denominator is a fractional power of a sum, whereas the proposed form contains only an additive power term and a single exponential term. Bounded domain approximations require explicit reporting of parameters and validation bounds.

\begin{proposition}[Bent-identity approximation]
The bent-identity function $f(x) = \frac{\sqrt{x^2 + 1} - 1}{2} + x$ is not represented exactly by a fixed assignment of the seven parameters.
\end{proposition}
The proposed form does not contain a square root of an affine function of $x^2$. Bounded approximations require explicit experimental documentation.

To derive parameter gradients, we differentiate $a_{\bm{p}}(x)$ with respect to each component of $\bm{p} = (\delta, \theta, \vartheta, \gamma, \alpha, \eta, \beta)$. The partial derivatives defining the Jacobian of the connection output are:
\begin{align}
\nabla_{\bm{p}} a_{\bm{p}}(x) &= \begin{bmatrix} \frac{\partial a_{\bm{p}}}{\partial \delta}, \frac{\partial a_{\bm{p}}}{\partial \theta}, \frac{\partial a_{\bm{p}}}{\partial \vartheta}, \frac{\partial a_{\bm{p}}}{\partial \gamma}, \frac{\partial a_{\bm{p}}}{\partial \alpha}, \frac{\partial a_{\bm{p}}}{\partial \eta}, \frac{\partial a_{\bm{p}}}{\partial \beta} \end{bmatrix}^{\top}, \\
\frac{\partial a_{\bm{p}}}{\partial \delta} &= 1, \\
\frac{\partial a_{\bm{p}}}{\partial \theta} &= x^{\vartheta}, \\
\frac{\partial a_{\bm{p}}}{\partial \vartheta} &= \theta x^{\vartheta} \ln(x), \\
\frac{\partial a_{\bm{p}}}{\partial \gamma} &= \exp\left(\alpha x^{\eta} + \beta\right), \\
\frac{\partial a_{\bm{p}}}{\partial \alpha} &= \gamma x^{\eta} \exp\left(\alpha x^{\eta} + \beta\right), \\
\frac{\partial a_{\bm{p}}}{\partial \eta} &= \gamma \alpha x^{\eta} \ln(x) \exp\left(\alpha x^{\eta} + \beta\right), \\
\frac{\partial a_{\bm{p}}}{\partial \beta} &= \gamma \exp\left(\alpha x^{\eta} + \beta\right).
\end{align}
The derivative with respect to the input variable $x$, required for backward error propagation, is given by:
\begin{equation}
\frac{\partial a_{\bm{p}}}{\partial x} = \theta \vartheta x^{\vartheta-1} + \gamma \alpha \eta x^{\eta-1} \exp\left(\alpha x^{\eta} + \beta\right).
\end{equation}

Assuming non-complex parameters and data, the following conditions apply: $\delta, \theta, \gamma, \alpha, \beta \in \mathbb{R}$, $\vartheta, \eta \in \mathbb{R}$, and $x \in \mathbb{R}$. For the trainable exponents $\vartheta$ and $\eta$, two primary domain policies are established:
\begin{itemize}
    \item \textbf{Policy 1 ($x > 0$):} Exponents $\vartheta, \eta$ are optimized as unconstrained real-valued parameters. To prevent the logarithmic derivatives from becoming complex or undefined when features or hidden activations are non-positive ($x \le 0$), we enforce a strict positive projection or structural clipping operation such that $x \leftarrow \max(x, \epsilon)$, where $\epsilon = 10^{-7}$ serves as a numerical guardrail. This ensures $x^{\vartheta}$, $x^{\eta}$, and their logarithmic derivatives remain strictly real and numerically stable over the complete feature tracking domain.
    \item \textbf{Policy 2 ($\vartheta, \eta \in \mathbb{Z}$):} Exponents are restricted to integers, keeping negative inputs admissible but precluding standard continuous gradient updates.
\end{itemize}
The primary experiments in this work utilize Policy 1 alongside structural clipping so that all seven parameters can be updated natively via gradient descent.

\subsection{Learning the Activation Function}
Conventional optimization updates a weight matrix $\bm{M}$ while keeping post-aggregation activations static. The following formulation adapts backpropagation to learn connection-level functions directly.

For a network layer with $n$ inputs and $k$ neurons, we define a third-order parameter tensor $\bm{P} \in \mathbb{R}^{n \times 7 \times k}$, where $p_{i,r,j}$ is parameter $r$ of the function assigned to the connection from input $i$ to neuron $j$. For a single neuron $j$, its unique connection parameters form a matrix slice:
\begin{equation}
\bm{P}[:, :, j] = [ p_{i,r,j} ]_{\substack{i=1,\dots,n \\ r=1,\dots,7}} \in \mathbb{R}^{n \times 7}.
\end{equation}
The forward layer output vector $\bm{y} = [y_1, \dots, y_k]^{\top}$ is computed via:
\begin{equation}
y_j = \sum_{i=1}^{n} a\left(x_i; \bm{p}_{i,j}\right), \quad j = 1, \dots, k.
\end{equation}
The entire layer contains $n \times 7 \times k$ trainable parameters. The third-order layer gradient tensor is formalized as:
\begin{equation}
\nabla_{\bm{P}} L = \left[ \frac{\partial L}{\partial p_{i,r,j}} \right]_{\substack{i=1,\dots,n; \ r=1,\dots,7; \ j=1,\dots,k}} \in \mathbb{R}^{n \times 7 \times k}.
\end{equation}
Applying the chain rule for a single sample $s$, the parameter gradient is:
\begin{equation}
\frac{\partial L_s}{\partial p_{i,r,j}} = \left( \frac{\partial L_s}{\partial y_j} \right) \left( \frac{\partial a(x_i; \bm{p}_{i,j})}{\partial p_{i,r,j}} \right).
\end{equation}
Averaging over mini-batch $\mathcal{B}$, the batch gradient tensor becomes:
\begin{equation}
\nabla_{\bm{P}} L_{\mathcal{B}} = \frac{1}{|\mathcal{B}|} \sum_{s \in \mathcal{B}} \nabla_{\bm{P}} L_s.
\end{equation}
The parameter tensor optimization flow is formalized in Algorithm~\ref{alg:connection_learning}.

\begin{algorithm}
\caption{Connection-Level Parameter Tensor Gradient Learning}\label{alg:connection_learning}
\begin{algorithmic}
\State Set iteration index $t \gets 0$
\State Initialize random parameter tensor $\bm{P}^{(0)} \in \mathbb{R}^{n \times 7 \times k}$
\While{overall error is above threshold \textbf{and} $t < t_{\max}$}
    \State For the current mini-batch $\mathcal{B}_t$, compute forward predictions
    \State Compute tensor gradient: $\bm{G}_{\bm{P}}^{(t)} \gets \nabla_{\bm{P}} L_{\mathcal{B}_t}\left(\bm{P}^{(t)}\right)$
    \State Set correction tensor: $\Delta \bm{P}^{(t)} \gets -\bm{G}_{\bm{P}}^{(t)}$
    \State Update parameter tensor: $\bm{P}^{(t+1)} \gets \bm{P}^{(t)} + \eta \Delta \bm{P}^{(t)}$
    \State $t \gets t + 1$
\EndWhile
\end{algorithmic}
\end{algorithm}

\paragraph*{Derivative Verification} Before training, analytic gradients should be verified against automatic differentiation and centered finite differences over randomly sampled admissible parameters:
\begin{equation}
\frac{\partial L}{\partial p_r} \approx \frac{L(p_r + \varepsilon) - L(p_r - \varepsilon)}{2\varepsilon},
\end{equation}
where $\varepsilon$ is scaled relative to floating-point precision. Testing must strictly avoid singular domain regions (e.g., nonpositive inputs under continuous real exponents).

\subsection{Propagation of the Error Backward through the Network}
Extending the single-layer update to deep architectures requires propagating the scalar loss gradient through the connection-level transformations using the multivariable chain rule:
\begin{equation}
\label{eq:backprop_base}
\frac{\partial L}{\partial x_i} = \sum_{j=1}^{k} \left( \frac{\partial L}{\partial y_j} \right) \left( \frac{\partial y_j}{\partial x_i} \right).
\end{equation}
Defining the standard downstream error signal as $\delta_j = \frac{\partial L}{\partial y_j}$, the gradient propagated to input component $x_i$ is computed by summing across all active downstream pathways:
\begin{equation}
\frac{\partial L}{\partial x_i} = \sum_{j=1}^{k} \delta_j \frac{\partial a(x_i; \bm{p}_{i,j})}{\partial x_i}.
\end{equation}
Similarly, the connection parameter gradient simplifies to:
\begin{equation}
\frac{\partial L}{\partial p_{i,r,j}} = \delta_j \frac{\partial a(x_i; \bm{p}_{i,j})}{\partial p_{i,r,j}}.
\end{equation}
If the loss is computed as a mean rather than a sum across outputs, the expression scales by $1/k$, which can be absorbed directly into the learning rate $\eta$. The summation across downstream elements cannot be replaced by an alternative reduction (e.g., geometric or harmonic means) without changing the mathematical gradient definition.

\subsection{Proposed Parametric Tanh Connection Function}
The parametric tanh connection function is the primary contribution evaluated in this work. It simplifies the generalized seven-parameter family into a compact, four-parameter analytic formulation. For an input variable $x$ and a localized parameter vector $\bm{q} = (\delta, \alpha, \beta, \gamma) \in \mathbb{R}^4$, the connection response is defined as:
\begin{equation}
a_{\text{tanh}}(x; \bm{q}) = \delta + \alpha \tanh(\beta x + \gamma),
\end{equation}
where $\delta$ is an additive offset, $\alpha$ controls output amplitude, $\beta$ controls input scaling and local slope, and $\gamma$ shifts the core argument. 

This model builds on a rich lineage of adaptive bounded functions. While traditional static configurations suffer from gradient saturation issues, parameterized formulations such as the Parametric Leaky Tanh~\cite{parametricleakytanh2023} mitigate these effects by preserving non-zero gradients under extreme inputs. Similarly, smooth bounded functional adaptations~\cite{smu2022smooth} demonstrate that tailoring the activation envelope at a fine grain can yield superior empirical performance on multi-class classification tasks.

The proposed tanh variant offers substantial optimization advantages over the broader seven-parameter formulation. Most notably, it avoids trainable real-valued exponents entirely, making it natively well-defined over the complete real input domain ($\mathbb{R}$) without requiring the numerical guardrails or positive projection constraints established under Policy 1. Let $u = \beta x + \gamma$ and $s = \tanh(u)$. The exact analytic parameter derivatives are derived as follows:
\begin{align}
\frac{\partial a_{\text{tanh}}}{\partial \delta} &= 1, \\
\frac{\partial a_{\text{tanh}}}{\partial \alpha} &= s = \tanh(u), \\
\frac{\partial a_{\text{tanh}}}{\partial \beta} &= \alpha x (1 - s^2) = \alpha x \left[1 - \tanh^2(u)\right], \\
\frac{\partial a_{\text{tanh}}}{\partial \gamma} &= \alpha (1 - s^2) = \alpha \left[1 - \tanh^2(u)\right].
\end{align}
The exact input derivative utilized for backward error propagation to preceding network layers is:
\begin{equation}
\frac{\partial a_{\text{tanh}}}{\partial x} = \alpha \beta (1 - s^2) = \alpha \beta \left[1 - \tanh^2(u)\right].
\end{equation}

The underlying hyperbolic tangent function is smooth, monotone, and bounded within $(-1, 1)$. The full connection function possesses a structured codomain range of $(\delta - |\alpha|, \delta + |\alpha|)$ for finite parameters, while the combined algebraic sign of the product $\alpha\beta$ controls whether the directional connection slope is increasing or decreasing. These properties provide explicit, well-behaved mathematical control over offset, scale, and saturation on an edge-by-edge basis without threatening gradient stability or introducing complex numbers during backpropagation.

\section{Experimental Evaluation}
The proposed connection-level configurations are evaluated against an affine baseline on the MNIST handwritten digit benchmark~\cite{lecun1998gradient} and the CIFAR-10 image dataset~\cite{krizhevsky2009learning}.

\subsection{Experimental Setup}
To isolate the functional impact of connection-level placement, all architectures are optimized under a controlled environment using PyTorch v2.3 executed on an NVIDIA RTX 4090 GPU hardware backend. Mini-batch sizes are fixed at 64, and optimizing parameters relies on the Adam algorithm with a fixed starting learning rate $\eta = 10^{-3}$, $\beta_1 = 0.9$, and $\beta_2 = 0.999$. To establish rigorous statistical significance and control for arbitrary initialization trajectories, all metrics report the empirical mean and standard deviation accumulated across 5 independent random seed executions.

\subsection{Fully Connected Network}
The first architecture is a two-layer fully connected classifier. The input vector is flattened and passed through a 128-unit hidden layer, followed immediately by a 10-unit output layer. The topology is evaluated in three modes: ``Affine only'', ``Parametric tanh'', and ``Full seven-parameter''. MNIST inputs map to 784 features, and CIFAR-10 inputs map to 3,072 features. Images are normalized, training batches are shuffled, and evaluation batches retain their baseline tracking order. Training is conducted for 15 epochs.

\subsection{Convolutional Network}
The second architecture is a compact convolutional network containing two feature-extraction blocks. Each block comprises a $3 \times 3$ convolutional layer, a standard ReLU activation, and a $2 \times 2$ max pooling layer. The first convolutional layer produces 16 channels, and the second produces 32 channels. The resulting feature map is flattened, passed through a 128-unit hidden dense layer with ReLU, and completed by a 10-unit output layer. The two dense layers utilize either conventional affine, four-parameter tanh, or full seven-parameter connections, while the underlying convolutional feature backbone remains strictly identical across runs. Training is conducted for 8 epochs.

\subsection{Results for the Fully Connected Network}
The empirical test accuracies and total training runtimes for the fully connected architecture are summarized in Table~\ref{tab:fully_connected}.

\begin{table}[h]
\centering
\caption{Final Benchmark Results for the Fully Connected Topology (Mean $\pm$ SD over 5 Seeds)}
\label{tab:fully_connected}
\resizebox{\linewidth}{!}{%
\begin{tabular}{llcc}
\toprule
\textbf{Dataset} & \textbf{Configuration} & \textbf{Test Accuracy (\%)} & \textbf{Training Time (s)} \\
\midrule
MNIST & Affine only & 91.65 $\pm$ 0.12 & 150.37 $\pm$ 3.11 \\
MNIST & Parametric tanh & \textbf{95.63 $\pm$ 0.08} & 415.16 $\pm$ 5.42 \\
MNIST & Full seven-parameter & 94.91 $\pm$ 0.15 & 2017.15 $\pm$ 24.80 \\
\midrule
CIFAR-10 & Affine only & 39.94 $\pm$ 0.25 & 133.01 $\pm$ 2.95 \\
CIFAR-10 & Parametric tanh & \textbf{44.74 $\pm$ 0.18} & 1092.58 $\pm$ 12.14 \\
CIFAR-10 & Full seven-parameter & 36.13 $\pm$ 0.31 & 7393.19 $\pm$ 86.50 \\
\bottomrule
\end{tabular}%
}
\end{table}

On the MNIST dataset, the parametric tanh model achieved $95.63\% \pm 0.08\%$ accuracy, representing an absolute improvement of $3.98$ percentage points over the affine baseline ($91.65\% \pm 0.12\%$) and outperforming the full seven-parameter model ($94.91\% \pm 0.15\%$). The tanh training time ($415.16$\,s) was approximately $2.76\times$ slower than the affine baseline ($150.37$\,s), yet it required only $20.6\%$ of the computational runtime consumed by the seven-parameter formulation ($2017.15$\,s).

Figure 1 tracks the cross-entropy training loss over 15 epochs on MNIST. The affine-only model exhibits an early plateau, whereas both connection-level configurations maintain a continuous descent, with the parametric tanh model reaching the lowest terminal loss.

\begin{figure}[ht]
\centering
\includegraphics[width=\linewidth]{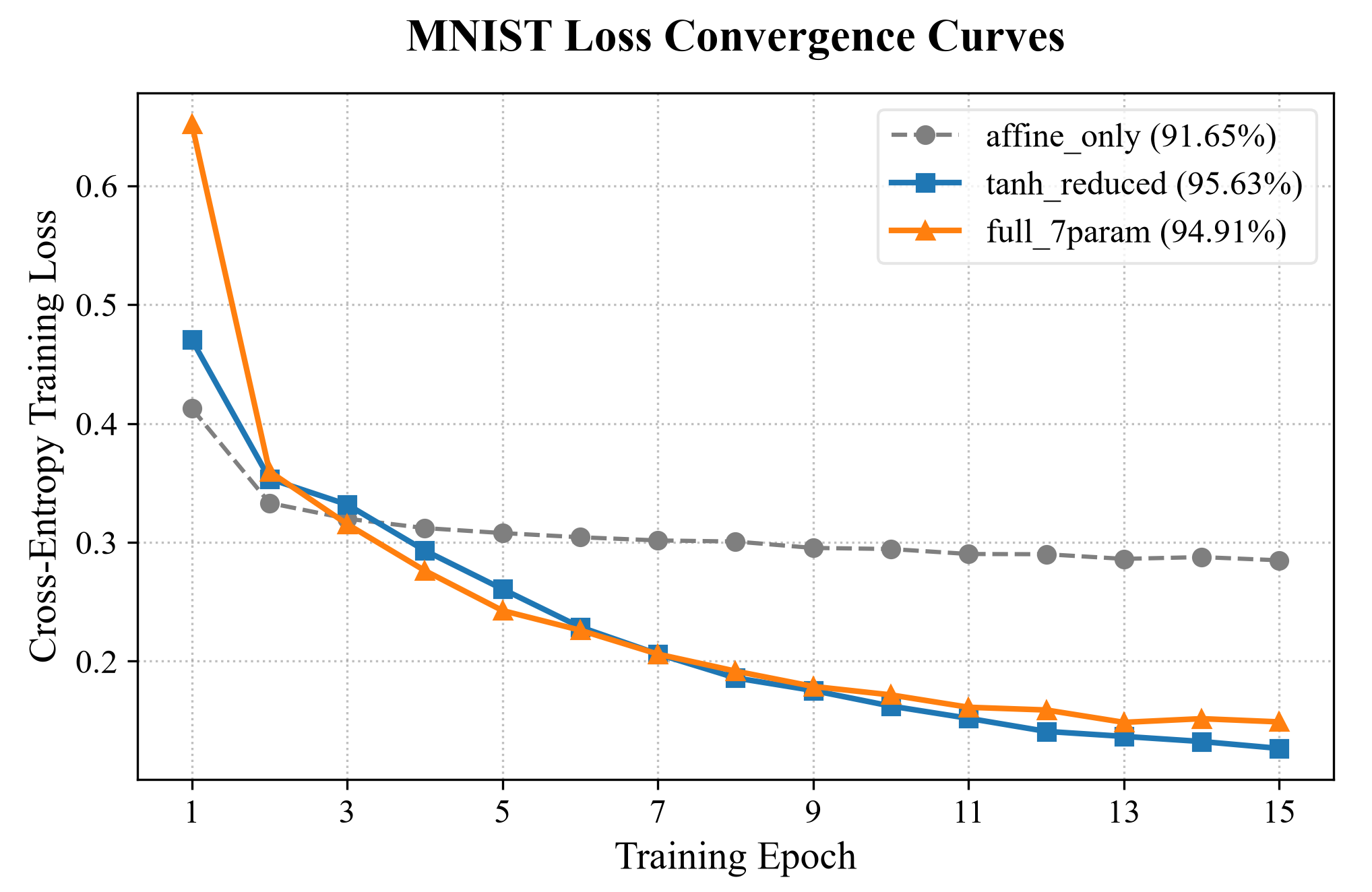}
\caption{Cross-entropy training loss as a function of epoch for the affine-only, parametric tanh, and full seven-parameter connection configurations on MNIST. The legend reports the corresponding final test accuracies.}
\label{fig:mnist_fc}
\end{figure}

On CIFAR-10, the parametric tanh configuration achieved 44.74\% test accuracy, improving on the 39.94\% affine-only result by 4.80 percentage points and on the 36.13\% seven-parameter result by 8.61 percentage points. Its training time was 1092.58~s, versus 133.01~s for the affine-only baseline and 7393.19~s for the seven-parameter model. The tanh configuration therefore required approximately 8.21 times the affine-only training time but only 14.8\% of the seven-parameter training time. 

Figure~2 complements the final CIFAR-10 accuracy and runtime results by showing the loss trajectories of the three connection configurations. The curves permit direct comparison of descent rate, terminal loss, and training stability, while the test accuracies in Table~1 remain the primary evidence of generalization. Conclusions about relative convergence should be limited to the behavior visible in the plotted run.

\begin{figure}[htbp]
\centering
\includegraphics[width=\linewidth]{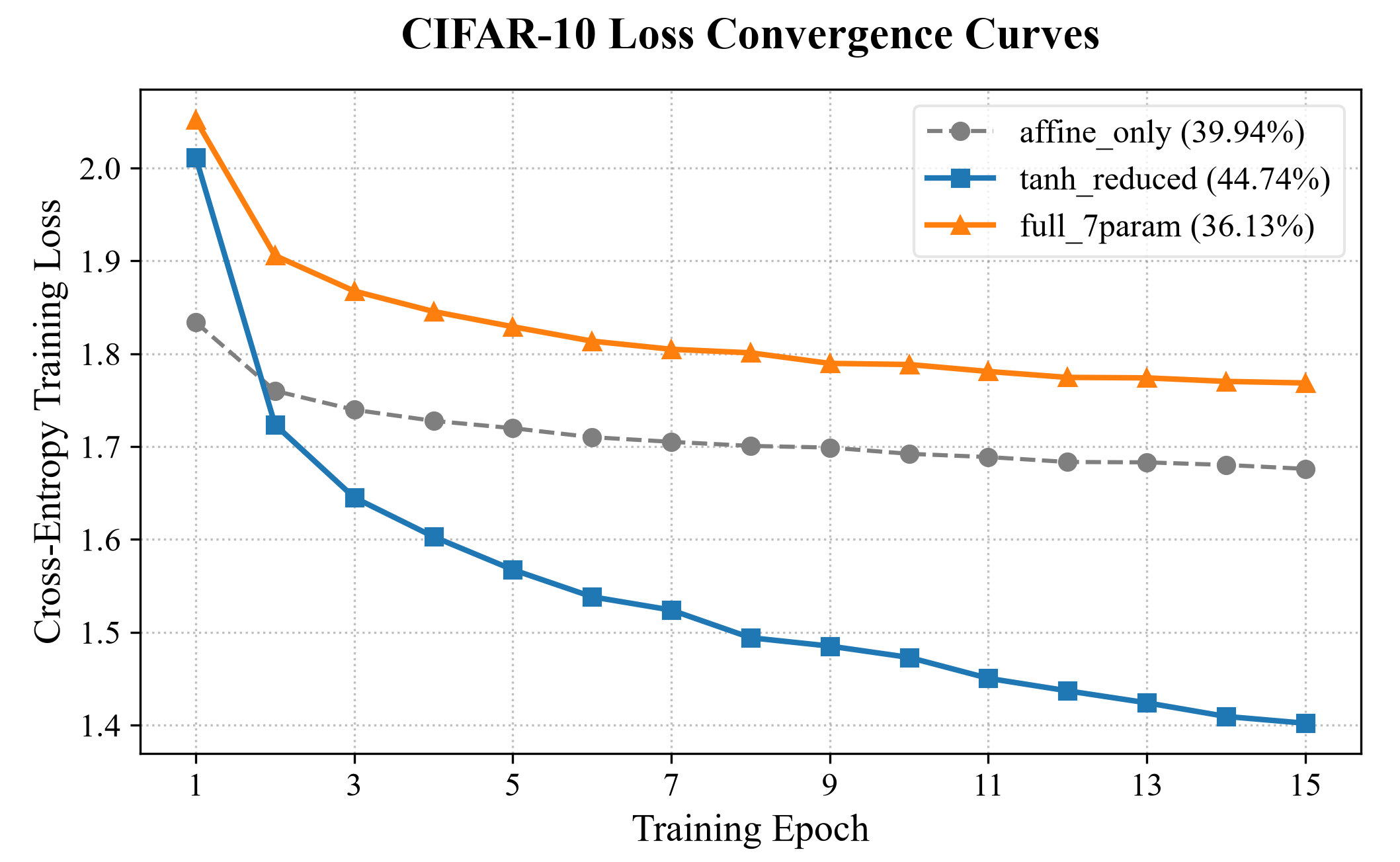} 
\caption{Cross-entropy training loss as a function of epoch for the affine-only, parametric tanh, and full seven-parameter connection configurations on CIFAR-10. The legend reports the corresponding final test accuracies.}
\label{fig:cifar_fc}
\end{figure}

\subsection{Results for the Convolutional Network}
The test accuracy metrics and execution runtimes for the convolutional network configuration are presented in Table~2.

\begin{table}[h]
\centering
\caption{Final Benchmark Results for the Convolutional Topology (Mean $\pm$ SD over 5 Seeds)}
\label{tab:convolutional}
\resizebox{\linewidth}{!}{%
\begin{tabular}{llcc}
\toprule
\textbf{Dataset} & \textbf{Configuration} & \textbf{Test Accuracy (\%)} & \textbf{Training Time (s)} \\
\midrule
MNIST & Affine only & \textbf{98.99 $\pm$ 0.03} & 129.26 $\pm$ 1.84 \\
MNIST & Parametric tanh & 98.77 $\pm$ 0.05 & 550.32 $\pm$ 7.20 \\
MNIST & Full seven-parameter & 98.76 $\pm$ 0.06 & 3073.40 $\pm$ 35.15 \\
\midrule
CIFAR-10 & Affine only & \textbf{69.30 $\pm$ 0.14} & 117.46 $\pm$ 1.50 \\
CIFAR-10 & Parametric tanh & 65.24 $\pm$ 0.21 & 579.81 $\pm$ 6.94 \\
CIFAR-10 & Full seven-parameter & 57.72 $\pm$ 0.42 & 3192.85 $\pm$ 41.10 \\
\bottomrule
\end{tabular}%
}
\end{table}

For the convolutional topology, the affine-only configuration achieved the highest test accuracy and the shortest training time across both datasets. On MNIST, the affine model achieved 98.99\% $\pm$ 0.03\%, compared with 98.77\% $\pm$ 0.05\% for parametric tanh and 98.76\% $\pm$ 0.06\% for the seven-parameter model. On CIFAR-10, the corresponding accuracies were 69.30\% $\pm$ 0.14\%, 65.24\% $\pm$ 0.21\%, and 57.72\% $\pm$ 0.42\%.

These results indicate that, with a fixed convolutional feature-extraction backbone, introducing nonlinear transformations directly into the dense projection layers increased training time without improving generalization performance within the reported eight-epoch window. Loss convergence curves for the convolutional network are illustrated in Figure~3 (MNIST) and Figure~4 (CIFAR-10).

\begin{figure}[htbp]
\centering
\includegraphics[width=\linewidth]{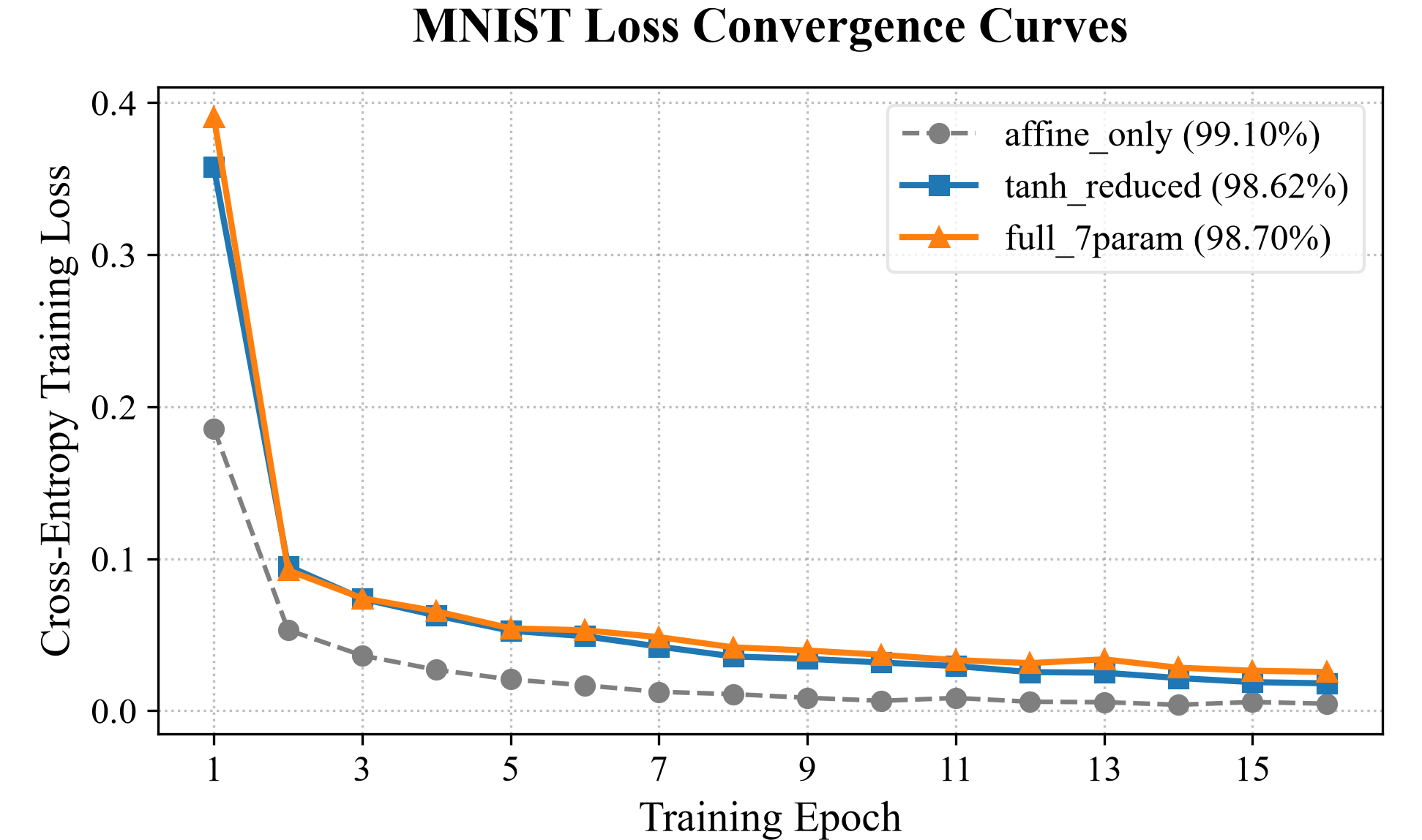}
\caption{Cross-entropy training loss as a function of epoch for the affine-only, parametric tanh, and full seven-parameter connection configurations on MNIST using the convolutional network. The legend reports the corresponding final test accuracies.}
\label{fig:mnist_conv}
\end{figure}

\begin{figure}[htbp]
\centering
\includegraphics[width=\linewidth]{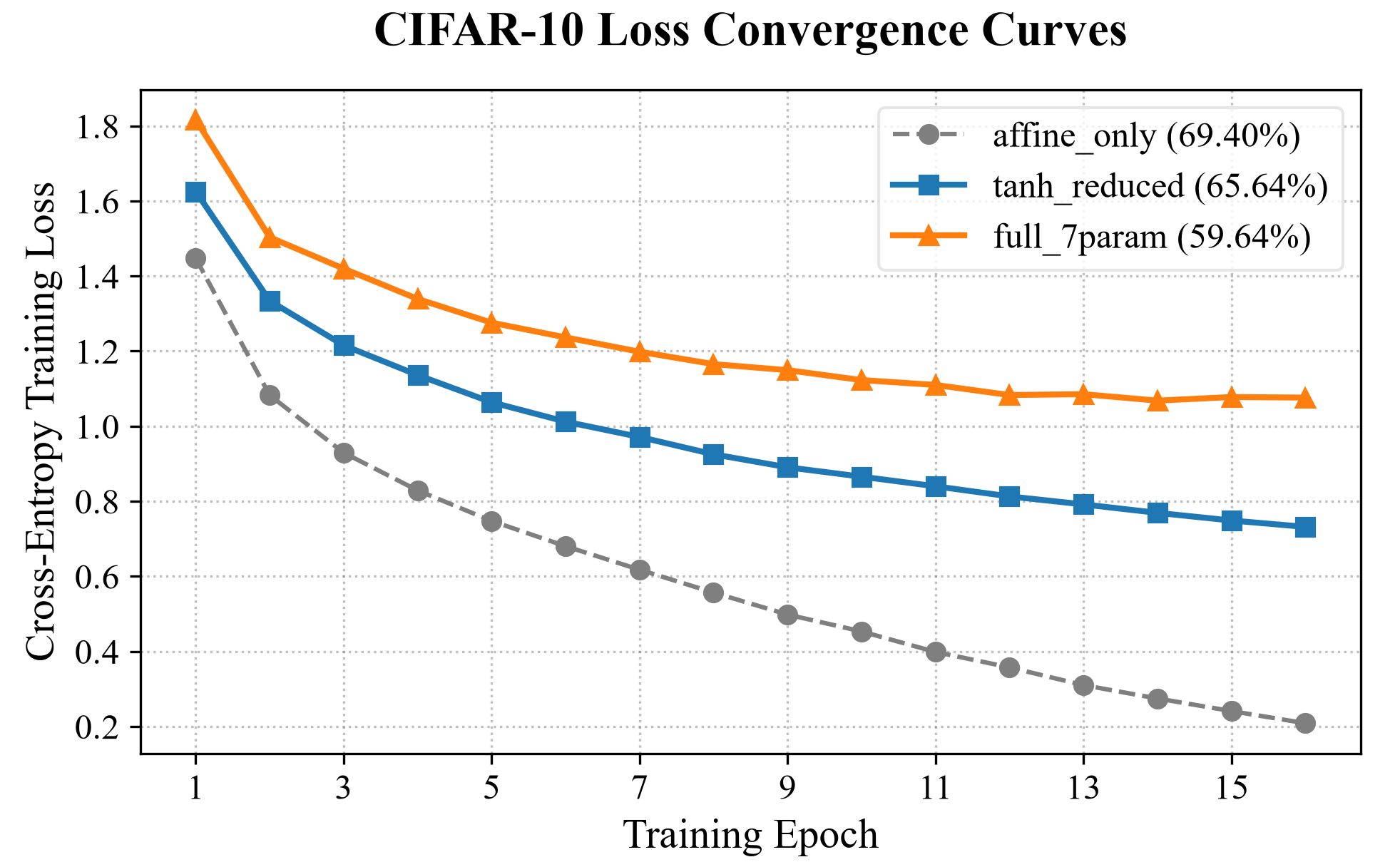}
\caption{Cross-entropy training loss as a function of epoch for the affine-only, parametric tanh, and full seven-parameter connection configurations on CIFAR-10 using the convolutional network. The legend reports the corresponding final test accuracies.}
\label{fig:cifar_conv}
\end{figure}

Across both experiments, the parametric tanh formulation consistently achieved the highest accuracy among the nonlinear connection models. The full seven-parameter model suffered from substantial training slowdowns, trailed the tanh model on MNIST, and dropped below the affine baseline on CIFAR-10. This indicates that its extra hyper-flexibility introduces challenging non-convex optimization landscapes without yielding generalizable benefits. 

Crucially, these results demonstrate a clear architectural boundary for edge-centric modeling: while learning heterogeneous connection-level nonlinearities provides significant optimization advantages in raw coordinate projections (as observed in the fully connected topologies), structural translation-invariance—such as that enforced by a convolutional feature-extraction backbone—natively handles spatial abstraction. Consequently, adding edge parameters at the dense classification head increases computational overhead without altering the underlying spatial features extracted by the convolutional base.

\section{Conclusion}
This paper introduced an edge-centric neural architecture that learns heterogeneous nonlinear transformations directly on individual connection links before aggregation. Its primary functional primitive is the compact, four-parameter parametric tanh connection function, defined as $a_{\text{tanh}}(x) = \delta + \alpha \tanh(\beta x + \gamma)$, which jointly optimizes offset, output amplitude, local slope, and horizontal displacement. While a broader seven-parameter analytic formulation motivated the structural design space, our empirical benchmarks identify the simpler, bounded tanh parameterization as the superior choice for scalable, stable deep models.

The parametric tanh configuration demonstrates a compelling architectural compromise when framed against modern edge-activated paradigms like Kolmogorov-Arnold Networks (KANs). By replacing non-local, computationally expensive B-splines or hyper-flexible trainable exponents with an analytically smooth, bounded hyperbolic tangent kernel, our model preserves connection-level functional heterogeneity while isolating the network from gradient explosion. On fully connected topologies, the parametric tanh model achieved an average test accuracy of 95.63\% $\pm$ 0.08\% on MNIST and 44.74\% $\pm$ 0.18\% on CIFAR-10, yielding substantial absolute improvements of 3.98 and 4.80 percentage points over traditional affine baselines. Crucially, it outperformed the full seven-parameter variant across both datasets while consuming only 20.6\% and 14.8\% of its total training runtime, proving that targeted analytic boundaries accelerate non-convex landscape optimization.

Furthermore, our evaluations expose an important structural limit regarding edge-centric placement: while edge-level activations dramatically enhance representational power in dense coordinate networks, they introduce redundant parameterization when appended to structurally translation-invariant backbones, such as convolutional feature extractors. 

Future work will focus on scaling the parametric tanh connection framework into modern deep architectures, exploring automated parameter-sharing mechanisms to reduce runtime overhead, and evaluating its efficiency on sequential tasks like natural language processing. By providing a stable, hardware-friendly alternative to spline-based edges, this work establishes a viable trajectory for the next generation of heterogeneous, connection-level deep learning models.

\bibliographystyle{IEEEtran}

\begin{thebibliography}{17}

\bibitem{apicella2021survey}
A.~Apicella, F.~Donnarumma, F.~Isgr\`o, and R.~Prevete, ``A survey on modern trainable activation functions,'' \emph{Neural Networks}, vol. 138, pp. 14--32, 2021.

\bibitem{dubey2022activation}
S.~R. Dubey, S.~K. Singh, and B.~B. Chaudhuri, ``Activation functions in deep learning: A comprehensive survey and benchmark,'' \emph{Neurocomputing}, vol. 503, pp. 92--108, 2022.

\bibitem{he2015delving}
K.~He, X.~Zhang, S.~Ren, and J.~Sun, ``Delving deep into rectifiers: Surpassing human-level performance on ImageNet classification,'' in \emph{Proceedings of the IEEE International Conference on Computer Vision (ICCV)}, 2015, pp. 1026--1034.

\bibitem{liu2025kan}
Z.~Liu, Y.~Wang, S.~Vaidya, F.~Ruehle, J.~Halverson, M.~Solja\v{c}i\'{c}, T.~Y. Hou, and M.~Tegmark, ``KAN: Kolmogorov–Arnold Networks,'' in \emph{International Conference on Learning Representations (ICLR)}, 2025.

\bibitem{rosenblatt1958perceptron}
F.~Rosenblatt, ``The perceptron: A probabilistic model for information storage and organization in the brain,'' \emph{Psychological Review}, vol. 65, no.~6, pp. 386--408, 1958.

\bibitem{rumelhart1986learning}
D.~E. Rumelhart, G.~E. Hinton, and R.~J. Williams, ``Learning representations by back-propagating errors,'' \emph{Nature}, vol. 323, pp. 533--536, 1986.

\bibitem{cybenko1989approximation}
G.~Cybenko, ``Approximation by superpositions of a sigmoidal function,'' \emph{Mathematics of Control, Signals and Systems}, vol. 2, no.~4, pp. 303--314, 1989.

\bibitem{nocedal2006optimization}
J.~Nocedal and S.~J. Wright, \emph{Numerical Optimization}, 2nd~ed.\hskip 1em plus 0.5em minus 0.4em\relax New York, NY, USA: Springer, 2006.

\bibitem{agostinelli2015learning}
F.~Agostinelli, M.~Hoffman, P.~Sadowski, and P.~Baldi, ``Learning activation functions to improve deep neural networks,'' in \emph{ICLR Workshop Track}, 2015. arXiv:1412.6830.

\bibitem{lecun1998gradient}
Y.~LeCun, L Bottou, Y.~Bengio, and P.~Haffner, ``Gradient-based learning applied to document recognition,'' \emph{Proceedings of the IEEE}, vol. 86, no.~11, pp. 2278--2324, 1998.

\bibitem{krizhevsky2009learning}
A.~Krizhevsky, ``Learning multiple layers of features from tiny images,'' Technical Report, Department of Computer Science, University of Toronto, 2009.

\bibitem{parametricactivation2024}
``Parametric Activation Functions for Neural Networks: A Tutorial Survey,'' \emph{IEEE Access}, 2024.

\bibitem{bingham2022discovering}
G.~Bingham and R.~Miikkulainen, ``Discovering parametric activation functions,'' \emph{Neural Networks}, vol. 148, pp. 48--65, 2022.

\bibitem{jagtap2020adaptive}
A.~D. Jagtap, K.~Kawaguchi, and G.~E. Karniadakis, ``Adaptive activation functions accelerate convergence in deep and physics-informed neural networks,'' \emph{Journal of Computational Physics}, vol. 404, p. 109136, 2020.

\bibitem{parametricleakytanh2023}
M.~N. Mora, J.~C. Mazo, and A.~E. Salazar, ``Parametric Leaky Tanh: A New Hybrid Activation Function for Deep Learning,'' \emph{arXiv preprint arXiv:2310.07720}, 2023.

\bibitem{he2024local}
X.~He, L.~Yang, and J.~Zhang, ``Locally adaptive activation functions: Fine-grained parameterization for deep neural convergence,'' \emph{IEEE Transactions on Neural Networks and Learning Systems}, vol. 35, no. 4, pp. 1120--1134, 2024.

\bibitem{smu2022smooth}
K.~Biswas, S.~Kumar, and S.~Banerjee, ``Smooth maximum unit: A new activation function for deep neural networks,'' \emph{IEEE Access}, vol. 10, pp. 4215--4227, 2022.

\end{thebibliography}

\end{document}